\documentclass[letterpaper, 10 pt, conference]{ieeeconf}  
\IEEEoverridecommandlockouts

\usepackage{dsfont}
\usepackage{color}
\definecolor{pink}{rgb}{0.58,0,0.83}
\definecolor{orange}{rgb}{1,0.5,0}
\definecolor{lightgreen}{rgb}{0.2, 0.8, 0.2}
\definecolor{lightyellow}{rgb}{0.84, 0.65, 0.13}

\usepackage[colorlinks,linkcolor=blue,anchorcolor=blue,citecolor=red]{hyperref}
\usepackage{amsmath,amsfonts}
\usepackage{algorithmic}
\newtheorem{theorem}{Theorem}
\usepackage[linesnumbered, ruled]{algorithm2e}
\usepackage{array}
\usepackage{subfigure}
\usepackage{makecell}
\newtheorem{problem}{Problem}
\usepackage{textcomp}
\usepackage{stfloats}
\usepackage{url}
\usepackage{verbatim}
\usepackage{graphicx}
\usepackage{amssymb}
\usepackage{dsfont}
\usepackage{bbding}
\usepackage{booktabs}
\usepackage{multirow}
\usepackage{tablefootnote}
\usepackage{cite}
\usepackage{xcolor}
\usepackage{threeparttable}
\newtheorem{assumption}{Assumption}
\newtheorem{definition}{Definition} 
\newtheorem{remark}{Remark}

\usepackage{arydshln}
\usepackage{multicol}
\usepackage{tablefootnote} 

\definecolor{ForestGreen}{RGB}{34,139,34}
\def\BibTeX{{\rm B\kern-.05em{\sc i\kern-.025em b}\kern-.08em
		T\kern-.1667em\lower.7ex\hbox{E}\kern-.125emX}}
\usepackage{balance}
\begin{document}
	\title{\bf Versatile Distributed Maneuvering with Generalized Formations using Guiding Vector Fields}
	\author{
		Yang Lu$^{1,*}$, Sha Luo$^{2,*}$, Pengming Zhu$^{1}$, Weijia Yao$^{3}$, Héctor García de Marina$^{4}$, Xinglong Zhang$^{1}$, Xin Xu$^{1}$
      \thanks{$^{1}$Yang Lu, Pengming Zhu, Xinglong Zhang, and Xin Xu are with the College of Intelligence Science and Technology, National University of Defense Technology, Changsha, 410073, China. $^{2}$Sha Luo is with the College of Information Science and Engineering, Hunan Normal University, Changsha, 410081, China. $^{3}$Weijia Yao is with the School of Robotics, Hunan University, Changsha, 410082, China. $^{4}$Héctor García de Marina is with the Department of Computer Engineering, Automation and Robotics, and with CITIC, Universidad de Granada, 18071 Granada, Spain. Corresponding authors: Weijia Yao, Xin Xu. (e-mail: {\tt\small weijia.yao.new@outlook.com}, {\tt\small xinxu@nudt.edu.cn})}
      \thanks{$^*$Yang Lu and Sha Luo contributed equally to this work.  }}
		%
	
	\maketitle
	\thispagestyle{empty}
	\pagestyle{empty}
	
	\begin{abstract}
        This paper presents a unified approach to realize versatile  distributed maneuvering with generalized formations. Specifically, we decompose the robots' maneuvers into two independent components, i.e., interception and enclosing, which are parameterized by two independent virtual coordinates. Treating these two virtual coordinates as dimensions of an abstract manifold, we derive the corresponding singularity-free guiding vector field (GVF), which, along with a distributed coordination mechanism based on the consensus theory, guides robots to achieve various motions (i.e., versatile maneuvering), including (a) formation tracking, (b) target enclosing, and (c) circumnavigation. Additional motion parameters can generate more complex cooperative robot motions. Based on GVFs, we design a controller for a nonholonomic robot model. Besides the theoretical results, extensive simulations and experiments are performed to validate the effectiveness of the approach.    
	\end{abstract}
	
	
	\definecolor{limegreen}{rgb}{0.2, 0.8, 0.2}
	\definecolor{forestgreen}{rgb}{0.13, 0.55, 0.13}
	\definecolor{greenhtml}{rgb}{0.0, 0.5, 0.0}
	
	\section{Introduction}
 Distributed coordinated maneuvering of multiple mobile robots plays a critical role in tasks such as environmental monitoring \cite{notomista2022multi}, collaborative transportation \cite{kabir2021efficient}, and target tracking \cite{corah2021scalable}. Various algorithms have been designed to achieve specific forms of robot coordinated maneuvering \cite{rosenfelder2021cooperative, ju2022mpc, ghommam2022design}, but a unified framework for realizing different forms of distributed robot maneuvering with generalized formations—referred to as \emph{versatile maneuvering}—has not been developed. These maneuvers include but are not limited to, formation tracking, (non-orbiting) target enclosing, and (orbiting) circumnavigation. In this work, we propose a unified framework based on guiding vector fields (GVFs) to achieve these maneuvers. Specifically, our approach combines different cooperative behaviors, such as interception and enclosing, each of which can be parameterized as a one-dimensional manifold (i.e., a desired path). The Cartesian product of these 1D manifolds creates a high-dimensional manifold  (for simplicity, we consider the two-dimensional case), which is used to derive a high-dimensional singularity-free GVF with virtual coordinates. Robots can achieve versatile distributed maneuvering by reaching a distributed consensus on these virtual coordinates. The guiding information from the GVFs is then employed to derive controllers for unicycle robots.
	
	\textbf{Related Work:} 
Distributed maneuvering control, including formation tracking and target enclosing, has been extensively studied \cite{li2012distributed, wang2017distributed, de2016distributed, xu2020moving, xu2023dual, hu2021bearing, shao2022distributed}. Formation tracking control ensures that each agent moves in coordination with others to preserve the desired formation shape while enclosing control enables a group of robots to surround a specific target \cite{de2016distributed}. Recent studies have focused on aspects such as ensuring even distribution of robots around the target \cite{xu2020moving}, forming formations for heterogeneous multi-robot systems \cite{xu2023dual}, estimating target positions \cite{hu2021bearing}, and dealing with uncertainties \cite{shao2022distributed}. Although significant progress has been made, a unified, robust, and scalable framework for adaptive and efficient distributed maneuvering control remains a challenge.

 In contrast to enclosing control, circumnavigation control drives robots to follow a predefined circular orbit around a target. Existing studies on circumnavigation focus on a variety of challenges, such as navigating moving 3D targets \cite{sinha20223}, accommodating robots with nonlinear dynamics \cite{yu2022optimal, luo2024optimal}, avoiding obstacles \cite{wang2024target, jiang2023safety}, and addressing communication-denied or GPS-denied situations \cite{yan2024distributed, zou2023circumnavigation, liu2023moving, wang2021mobile}. However, most of these methods focus on circular orbits, and the generalization to arbitrary orbits has not been fully explored. Furthermore, most studies restrict robots to moving along a single circular orbit around a target, with few addressing the possibility of robots circumnavigating multiple different orbits in a distributed manner.

Guiding vector fields (GVFs) have emerged as a powerful tool for distributed motion control in multi-robot systems \cite{frew2008coordinated, de2017circular, nakai2013vector, pimenta2013decentralized, 9969449}. For example, Hector et al. \cite{de2017circular} proposed a GVF-based circular formation control approach for multiple unmanned aerial vehicles (UAVs) to track parameterized circular paths. In \cite{9969449}, virtual coordinates are introduced to derive singularity-free GVFs, enabling robots to navigate surfaces and perform complex motion tasks. The robots' behaviors are coordinated through consensus theory \cite{mesbahi2010graph}. These methods have demonstrated computational efficiency, robustness, and scalability in coordinated motion control tasks.

Despite these successes, existing studies exploring vector fields in orbiting control tasks remain limited. For example, in \cite{zhong2019circumnavigation}, an orthogonal vector field method is used to guide networked robots to circumnavigate a dynamic target in 3D. However, it is constrained by a fixed orientation for planar motion, limiting its adaptability to complex trajectories. In \cite{wilhelm2019circumnavigation}, UAVs are guided by vector fields to circumnavigate a moving target while avoiding obstacles, but the approach is restricted to circular orbits of fixed radii. Similarly, \cite{muslimov2022application} presents a decentralized GVF-guided genetic approach for 2D trajectory optimization, but the orbits are again limited to constant radii, restricting its applicability to more dynamic environments.

	\textbf{Contributions:} We propose a unified framework based on guiding vector fields to realize versatile distributed maneuvering with generalized formations, including formation tracking, target enclosing/surrounding, and circumnavigation, in $\mathbb{R}^2$ and $\mathbb{R}^3$. We generalize the maneuvering orbits to (possibly multiple) arbitrary closed curves while an arbitrary number of robots' motion is coordinated via low communication bandwidth (every two robots only transmit two virtual scalar coordinates) without a centralized node,  broadening the real-world applications. Different from typical estimation techniques, our approach estimates the target's speed through the coordination around the target and the robot's actual velocities. 
    Moreover, based on the GVFs, we design a controller algorithm for a nonholonomic robot model. Extensive simulation and experimental results demonstrate the effectiveness of the framework. 
	
	\textbf{Notations:} The notation $\mathbb{Z}_i^j$ denotes the integer set $\{m\in\mathbb{Z}:i\le m\le j\}$. 
 Symbol $\wedge(\cdot)$ denotes the wedge product operation\cite[Eq.5]{yao2021singularity}. Given $N$ robots, the $j$-th coordinate of the $i$-th robot is denoted by $x_j^{[i]} \in \mathbb{R}$, where $i\in\mathbb{Z}_1^N$. Namely, the superscript $[i]$ denotes the ID of a robot while the subscript $j$ denotes the $j$-th entry of a related vector quantity (e.g., position). 
	The interaction topology of $N$ robots is represented by an \emph{undirected graph} $\mathcal{G}=(\mathcal{V},\mathcal{E})$, where the vertex set $\mathcal{V}=\{1,\cdots,N\}$ and the edge set $\mathcal{E}\subseteq\mathcal{V}\times\mathcal{V}$. The set of neighbor agents of the $i$-th robot is denoted by $\mathcal{N}_i:=\{j\in\mathcal{V}:(i,j)\in\mathcal{E}\}$. See \cite{mesbahi2010graph} for more details.
	
	\section{Preliminaries}\label{PreliminariesAndProblemFormulation}
	This section presents the preliminaries on distributed guiding vector field (DGVF) theory for multi-robot surface navigation, which was proposed in \cite{9969449} initially for multiple robots to navigate on a general two-dimensional manifold (i.e., surface) with motion coordination (e.g., formation maneuvering on a torus). The theory starts with a parametric description of a surface for robots to maneuver on. Precisely, the surface $\mathcal{S}^{[i]} \subseteq \mathbb{R}^n$ for the $i$-th robot is described by\footnote{For precision, the notation includes superscripts, but they can be ignored during reading as they only indicate which robot the variable corresponds to.}
	\begin{equation}\label{para_sur_P}
		x_1^{[i]} = h_1^{[i]}(s_1^{[i]},s_2^{[i]}), \cdots, x_n^{[i]}=h_n^{[i]}(s_1^{[i]},s_2^{[i]}),
	\end{equation}
	where $i\in\mathbb{Z}_1^N$, $x_j^{[i]}$ is its $j$-th coordinate, $s_1^{[i]},s_2^{[i]}\in\mathbb{R}$ are two parameters of the surface, and $h_j^{[i]}:\mathbb{R}^2 \to \mathbb{R}$ is the twice continuously differentiable parametric function for the surface. For $i\in\mathbb{Z}_1^N$, we introduce variables $w_1^{[i]}$ and $w_2^{[i]}\in\mathbb{R}$ to replace $s_1^{[i]}$ and $s_2^{[i]}$, respectively, and then the $i$-th robot's generalized coordinate is written as 
	$$\boldsymbol{\xi^{[i]}}:=(x_1^{[i]},\cdots,x_n^{[i]},w_1^{[i]},w_2^{[i]})\in\mathbb{R}^{n+2}$$
    %
   and $\phi_j^{[i]}(\boldsymbol{\xi^{[i]}}) := x_j^{[i]}-h_j^{[i]}(w_1^{[i]},w_2^{[i]})$ for $j \in \mathbb{Z}_1^n$. Then according to \cite{yao2021singularity}, the high dimensional \emph{surface-navigation vector field} $\boldsymbol{{ }^{\mathrm{sf}} \chi^{[i]}(\xi^{[i]})}$ is obtained by
	\begin{equation}
		\begin{aligned}
			\boldsymbol{{ }^{\mathrm{sf}} \chi^{[i]}(\xi^{[i]})}=&\wedge(\boldsymbol{\nabla\phi_1^{[i]}}(\boldsymbol{\xi^{[i]}}),\cdots,\boldsymbol{\nabla\phi_n^{[i]}}(\boldsymbol{\xi^{[i]}}))\\&-\sum_{j=1}^{n} k_j\phi_j^{[i]}(\boldsymbol{\xi^{[i]}})\boldsymbol{\nabla\phi_j^{[i]}}(\boldsymbol{\xi^{[i]}}),
		\end{aligned}
	\end{equation}
	where $k_j>0$ is the coefficient, and $\nabla\boldsymbol{\phi_j^{[i]}}$ is the gradient of $\phi_j^{[i]}$ with respect to the generalized coordinate $\boldsymbol{\xi^{[i]}}$.
	By adopting the consensus control algorithm\cite{ren2008distributed}, the motion coordinating part is given as below:
	\begin{equation}\label{CR_sur}
		\begin{aligned}
			{}^{\mathbf{cr}}\boldsymbol{\chi_1^{[i]}}(\boldsymbol{w^{[\cdot]}_1})=(0,\cdots,0,c_1^{[i]}(\boldsymbol{w_1^{[\cdot]}}),0)^\top,\\
			{}^{\mathbf{cr}}\boldsymbol{\chi_2^{[i]}}( \boldsymbol{w^{[\cdot]}_2})=(0,\cdots,0,0,c_2^{[i]}(\boldsymbol{w_2^{[\cdot]}}))^\top.
		\end{aligned}
	\end{equation}
 In this equation, $c_1^{[i]}(\boldsymbol{w_1^{[\cdot]}})$ and $c_2^{[i]}(\boldsymbol{w_2^{[\cdot]}})$ are formulated as follows:
	\begin{equation}\label{motion_coordination_parts_sur}
		\begin{array}{c}
			\begin{aligned}
				c_1^{[i]}(t,\boldsymbol{w^{[\cdot]}_1})&=-\sum_{j \in \mathcal{N}_i}(w_1^{[i]}(t)-w_1^{[j]}(t)-\Delta_1^{[i, j]}
				),\\
				c_2^{[i]}(t,\boldsymbol{w^{[\cdot]}_2})&=-\sum_{j \in \mathcal{N}_i}(w_2^{[i]}(t)-w_2^{[j]}(t)-\Delta_2^{[i, j]}
				),
			\end{aligned}
		\end{array}
	\end{equation}
	where $(i,j)\in\mathcal{E}, \Delta_1^{[i,j]},\Delta_2^{[i,j]}$ are the desired parametric differences between the virtual coordinates of the $i$-th and $j$-th robots. By stacking $\Delta_1^{[i,j]},\Delta_2^{[i,j]}$ as vectors, we define $\boldsymbol{\Delta_1^*}, \boldsymbol{\Delta_2^*} \in \mathbb{R}^{|\mathcal{E}|}$, respectively. 
	We write Eq.~\eqref{motion_coordination_parts_sur} compactly and have 
	\begin{equation*}
		\begin{aligned}
			\boldsymbol{c_1^{[\cdot]}}(\boldsymbol{w_1^{[\cdot]}})=-L(\boldsymbol{w_1^{[\cdot]}}-\boldsymbol{w_1^{*}})=-L\boldsymbol{\tilde{w}_1^{[\cdot]}},\\
			\boldsymbol{c_2^{[\cdot]}}(\boldsymbol{w_2^{[\cdot]}})=-L(\boldsymbol{w_2^{[\cdot]}}-\boldsymbol{w_2^{*}})=-L\boldsymbol{\tilde{w}_2^{[\cdot]}},
		\end{aligned}
	\end{equation*}
	where $\boldsymbol{c_1^{[\cdot]}}(\boldsymbol{w_1^{[\cdot]}})=(c_1^{[1]}(\boldsymbol{w_1^{[\cdot]}}),\cdots,c_1^{[N]}(\boldsymbol{w_1^{[\cdot]}}))$, $\boldsymbol{c_2^{[\cdot]}}(\boldsymbol{w_2^{[\cdot]}})=(c_2^{[1]}(\boldsymbol{w_2^{[\cdot]}}),\cdots,c_2^{[N]}(\boldsymbol{w_2^{[\cdot]}}))$, $\boldsymbol{w_1^*}=({w_1^{[1]}}^*,\cdots,{w_1^{[N]}}^*)$ and $\boldsymbol{w_2^*}=({w_2^{[1]}}^*,\cdots,{w_2^{[N]}}^*)$ are reference configurations for designing the desired parameteric differences, $L=L(\mathcal{G})$ is the Laplacian
	matrix encoding the communication topology among robots, and $\boldsymbol{\tilde{w}_1^{[\cdot]}}=\boldsymbol{w_1^{[\cdot]}}-\boldsymbol{w_1^{*}}$, $\boldsymbol{\tilde{w}_2^{[\cdot]}}=\boldsymbol{w_2^{[\cdot]}}-\boldsymbol{w_2^{*}}$.

	Finally, the DGVF is obtained by a weighted sum of \emph{surface-navigation} and \emph{coordination} parts\cite{9969449}, i.e.,
	\begin{equation}\label{GVF_sur}
		\begin{aligned}
			\boldsymbol{\mathfrak{X}^{[i]}}(\boldsymbol{\mathcal{\xi}^{[i]}})=\boldsymbol{{}^{\mathrm{sf}} \chi^{[i]}}(\boldsymbol{\xi^{[i]}})+k_{c1}{}^{\mathbf{cr}}\boldsymbol {\chi_1^{[i]}}( \boldsymbol{w_1^{[\cdot]}})+k_{c2}{}^{\mathbf{cr}}\boldsymbol {\chi_2^{[i]}}(\boldsymbol{w_2^{[\cdot]}}),
		\end{aligned}
	\end{equation}	
	where $k_{c1},k_{c2}>0$ are coefficients. The above DGVF guides a multi-robot system to achieve motion coordination on the desired parametric surface $\mathcal{S}^{[i]}$.
	
	In this paper, we present a novel perspective on this theory and its rigorous underpinning: \emph{rather than building the theory from a concrete, specific surface $\mathcal{S}^{[i]}$, we use the two variables $w_1$ and $w_2$ to describe two behaviors}; e.g., target interception and target enclosing, leading to versatile and distributed maneuvering (these two behaviors do result in a new surface, albeit abstract). Following this perspective, one can further include more behaviors and generate a manifold of higher dimensions, and it is straightforward to develop the corresponding DGVF \emph{without} singularities and distributed algorithms for accomplishing these tasks/behaviors.  In the sequel, we will elaborate on the two-variable case for simplicity.
    
	\section{Design of Guiding Vector Fields for Multi-Robot Systems}\label{distributed_leader_following_GVF}
	This section first introduces the definition of the combined behavior, illustrating examples with different values of desired virtual coordinates in TABLE~\ref{illustration_table}, and then formulates the problem of versatile maneuvering. A DGVF is designed, and theoretical analyses are conducted to address the problem.
 
 	\begin{figure}[!htbp]
		\centering\includegraphics[width=3.5in]{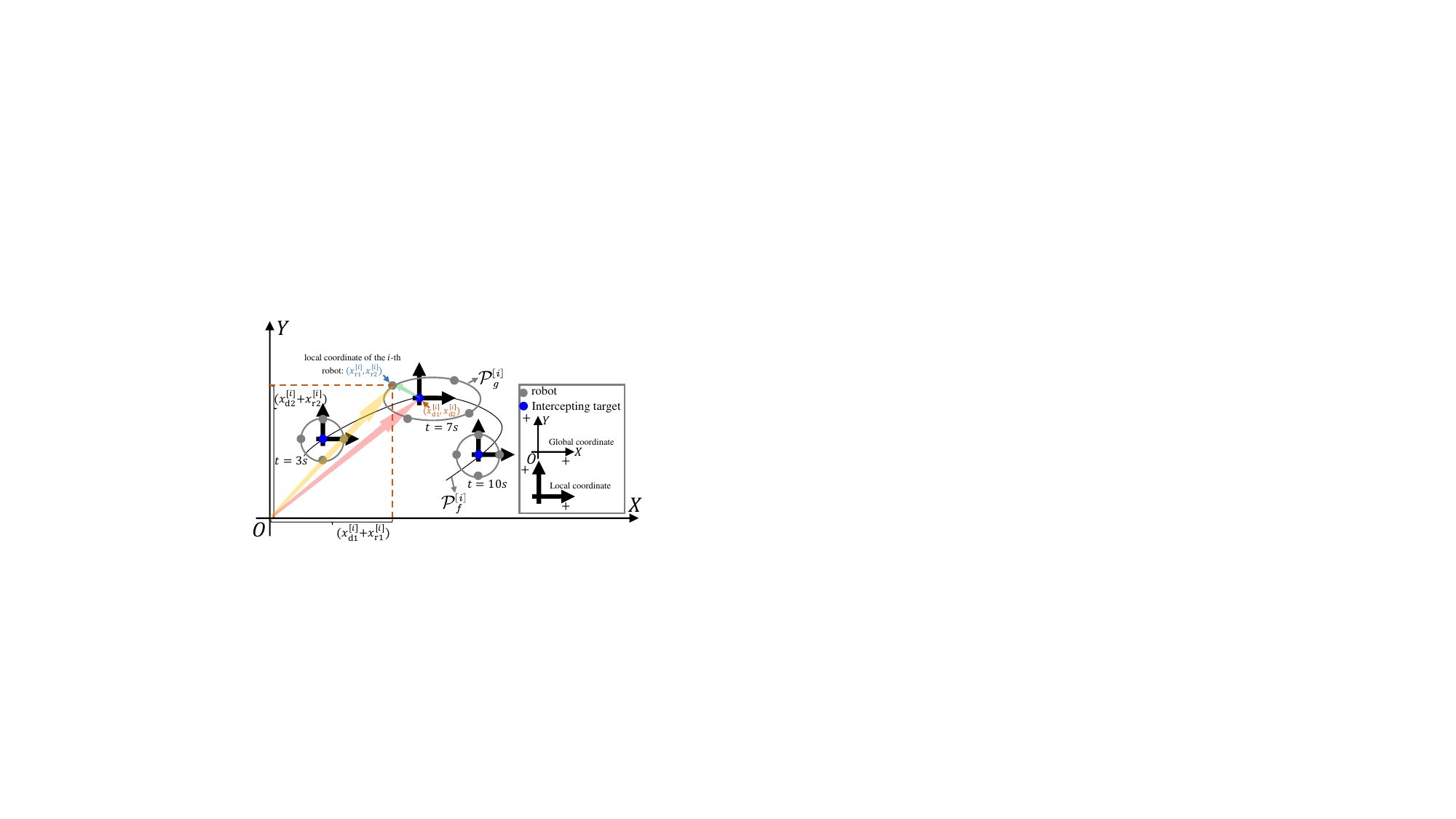}
		\caption{An illustration of versatile maneuvering: taking the $i$-th robot's at $t=7\mathrm{s}$ as an example. The red and green arrows depict the interception and enclosing behaviors, respectively, and the yellow arrow denotes the combined result for realizing versatile maneuvering.}
		\label{fig_illustration}
	\end{figure}
	\subsection{Fundamental definition and problem formulation}
	Suppose that $N$ robots are required to  achieve versatile and distributed maneuvers such as formation tracking, target enclosing (i.e., surrounding), or circumnavigation. We decompose these maneuvers into two independent behaviors; i.e., interception and enclosing. Therefore, this paper first presents the definition of the combined intercepting and enclosing behavior and then formulates the problem.
	\subsubsection{Definition of the combined behavior}
	\emph{Regarding the interception behavior}, we perceive it as requiring a reference point of the multi-robot system (e.g., the center of the robot formation) to follow a desired path parameterized by
	\begin{equation}\label{para_path_Q}
		x_{\mathrm{d}1}^{[i]} = f_1^{[i]}(s_1^{[i]}), \cdots, x_{\mathrm{d}n}^{[i]}=f_n^{[i]}(s_1^{[i]}),
	\end{equation}
	where $s_1^{[i]}\in\mathbb{R}$ is the parameter, and $f_j^{[i]}:\mathbb{R}^n\rightarrow\mathbb{R}$ is twice continuously differentiable. For $i\in\mathbb{Z}_1^N$, we introduce variable $w_1^{[i]}\in\mathbb{R}$ to replace parameter $s_1^{[i]}$ and the desired higher-dimensional path for the \(i\)-th robot in the interception behavior is given by
    \begin{equation*}\label{para_path_P_pro1}
		\mathcal{P}_f^{[i]}:=\{\boldsymbol{\xi_\mathrm{d}^{[i]}}\in\mathbb{R}^{n+1}:\phi_{\mathrm{d}1}^{[i]}(\boldsymbol{\xi^{[i]}_\mathrm{d}})=0, \cdots, \phi_{\mathrm{d}n}^{[i]}(\boldsymbol{\xi^{[i]}_\mathrm{d}})=0\},
	\end{equation*}
	where $\boldsymbol{\xi^{[i]}_\mathrm{d}}=(x_{\mathrm{d}1}^{[i]},\cdots,x_{\mathrm{d}n}^{[i]},w_1^{[i]})^\top\in\mathbb{R}^{n+1}$ and $\phi^{[i]}_{\mathrm{d}j}(\boldsymbol{\xi^{[i]}_\mathrm{d}})$ is defined as follows:
    \begin{equation*}
        \phi_{\mathrm{d}j}^{[i]}(x_{\mathrm{d}1}^{[i]},\cdots,x_{\mathrm{d}n}^{[i]},w_1^{[i]})=x_{\mathrm{d}j}^{[i]} - f_j^{[i]}(w_1^{[i]}).
    \end{equation*}
    
	\emph{Enclosing behavior} requires robots to either distribute (static formation) or move (non-static circumnavigation) along a closed parametric path, centered at $(x_{\mathrm{d}1}^{[i]},\cdots,x_{\mathrm{d}n}^{[i]})$ for the $i$-th robot and is parameterized by
	\begin{equation}\label{para_path_Q}
		x_{\mathrm{r}1}^{[i]} = g_1^{[i]}(s_2^{[i]}), \cdots, x_{\mathrm{r}n}^{[i]}=g_n^{[i]}(s_2^{[i]}),
	\end{equation}
	where $s_2^{[i]}\in\mathbb{R}$ is the parameter, and $g_j^{[i]}$ is twice continuously differentiable. We introduce a variable $w_2^{[i]}\in\mathbb{R}$ to replace $s_2^{[i]}$, for $i\in\mathbb{Z}_1^N$. The $i$-th robot's higher dimensional desired path of the enclosing behavior is described by
     \begin{equation*}
        \mathcal{P}_g^{[i]}:=\{\boldsymbol{\xi_\mathrm{r}^{[i]}}\in\mathbb{R}^{n+1}:\phi_{\mathrm{r}1}^{[i]}(\boldsymbol{\xi_\mathrm{r}^{[i]}})=0, \cdots, \phi^{[i]}_{\mathrm{r}n}(\boldsymbol{\xi_\mathrm{r}^{[i]}})=0\},
    \end{equation*}
	where $\boldsymbol{\xi_\mathrm{r}^{[i]}}:=(x_{\mathrm{r}1}^{[i]},\cdots,x_{\mathrm{r}n}^{[i]},w_2^{[i]})\in\mathbb{R}^{n+1}$ and $\phi_{\mathrm{r}j}^{[i]}(\boldsymbol{\xi_\mathrm{r}^{[i]}})$ is defined as follows:
    \begin{equation*}
        \phi_{\mathrm{r}j}^{[i]}(x_{\mathrm{r}1}^{[i]},\cdots,x_{\mathrm{r}n}^{[i]},w_2^{[i]})=x_{\mathrm{r}j}^{[i]} - g_j^{[i]}(w_2^{[i]}).
    \end{equation*}

	Denote the global coordinate of the $i$-th robot as $(x_1^{[i]},\cdots,x_n^{[i]})^\top, i\in\mathbb{Z}_1^N$. The definition of the combined behavior is presented as follows. 
	
	\begin{definition}\label{definition_1d_sur}
 The combined behavior of interception and enclosing is achieved if $\lim _{t \rightarrow \infty}\left\|(\phi_1^{[i]}(t),\cdots,\phi_n^{[i]}(t))\right\| = 0$, where $\phi_j^{[i]}(t):=x_j^{[i]}(t)-x_{\mathrm{d}j}^{[i]}-x_{\mathrm{r}j}^{[i]}$ for $i\in\mathbb{Z}_1^N$ and  $j\in\mathbb{Z}_1^n$.
	\end{definition}
 
    Fig.~\ref{fig_illustration} demonstrates versatile maneuvering, consisting of two independent intercepting and enclosing behaviors. By assigning different values to \(\boldsymbol{\Delta_1^*}\) and \(\boldsymbol{\Delta_2^*}\), a variety of distributed maneuvers can be realized, as shown in TABLE~\ref{illustration_table}.
	\begin{table}[!htbp]
		\begin{threeparttable}[b]
			\centering
			\caption{An illustration of setting different values for $\boldsymbol{\Delta_1^*}$ and $\boldsymbol{\Delta_2^*}$.}\label{illustration_table}
			\begin{tabular}{m{2.5cm}<{\centering}|m{5.3cm}<{\centering} m{4cm}<{\centering} m{4cm}<{\centering}}
				\hline\hline
				\textbf{Cases}&\textbf{Examples}
				\\\hline
				\underline{\textit{Case 1}:} $\boldsymbol{\Delta}_1^{*}=\boldsymbol{0}$ and \emph{evenly} distributed phase angles $\boldsymbol{\Delta_2^*}$ &\vspace{1mm} \includegraphics[width=0.25\textwidth]{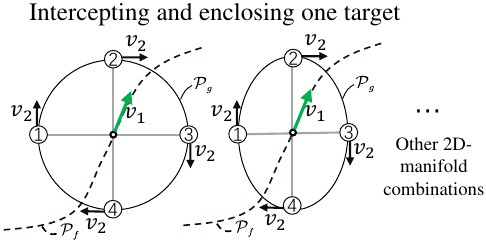}\\
				\hline
				\underline{\textit{Case 2}:} $\boldsymbol{\Delta}_1^{*}=\boldsymbol{0}$ and \emph{unevenly} distributed phase angles $\boldsymbol{\Delta_2^*}$&\vspace{1mm}\includegraphics[width=0.25\textwidth]{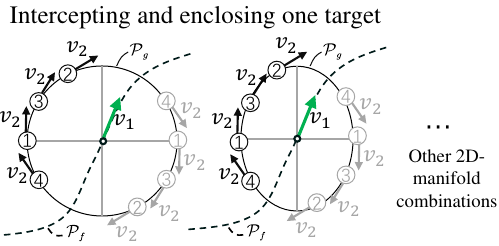}\\
				\hline
				\underline{\textit{Case 3}:} $\boldsymbol{\Delta}_1^{*}\neq\boldsymbol{0}$ and \emph{evenly} distributed phase angles $\boldsymbol{\Delta_2^*}$&\vspace{1mm}\includegraphics[width=0.3\textwidth]{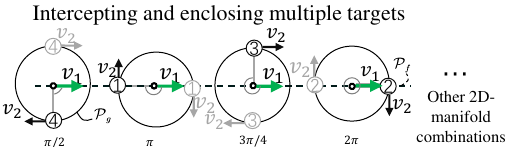}\\ 
				\hline
				\underline{\textit{Case 4}:} $\boldsymbol{\Delta}_1^{*}\neq\boldsymbol{0}$ and \emph{unevenly} distributed phase angles $\boldsymbol{\Delta_2^*}$&\vspace{1mm}\includegraphics[width=0.3\textwidth]{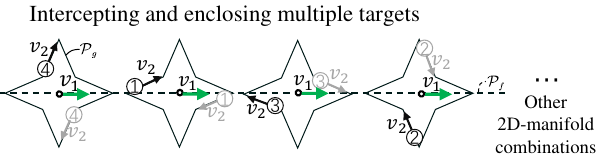}\\ 
				\hline\hline
			\end{tabular}
		\end{threeparttable}
	\end{table}
	\subsubsection{Problem formulation of versatile distributed maneuvering}
    To realize versatile distributed maneuvering, the idea is to define a composite 2D manifold to derive the DGVFs\cite{9969449}.	
     Since two virtual coordinates are included for obtaining the desired 2D manifold, the dimension of the \emph{generalized coordinate} is $n+2$.  We formulate the problem as follows:

	\begin{problem}\label{problem1}
		Design the DGVF $\boldsymbol{\mathfrak{X}^{[i]}}:\mathbb{R}^{n+2}\rightarrow\mathbb{R}^{n+2}$ in~\eqref{GVF_sur} for $i\in\mathbb{Z}_1^N$, such that the vector field can drive robots to fulfill the following objectives:
		\begin{itemize}
			\item[i)] (\emph{Interception and Enclosing}) The Euclidean norm of the $i$-th robot's following error to its desired composite 2D manifold tends to zero eventually; Namely, $\lim_{t\rightarrow\infty}\lVert(\phi_1^{[i]}(t),\cdots,\phi_n^{[i]}(t))\rVert=0$ for $i\in\mathbb{Z}_1^N$.
			
			\item[ii)] (\emph{Motion Coordination}) For any robots $i$ and $j$, $(i,j)\in\mathcal{E}$, they can communicate with each other to coordinate their motions. Specifically, 1) the virtual coordinates converge to the desired values. Namely, as $t\rightarrow \infty$, one has $w_1^{[i]}(t)-w_1^{[j]}(t)-\Delta_1^{[i, j]}\rightarrow0, w_2^{[i]}(t)-w_2^{[j]}(t)-\Delta_2^{[i, j]}\rightarrow0$; 2) the parametric speed $\dot{w}_1^{[i]}$ converges to the desired interception speed $\dot{w}_1^*$, and robots finally reach the desired enclosing speed $\dot{w}_2^*$. Namely, $\lim_{t\rightarrow\infty}\dot{w}_1^{[i]}(t)=\dot{w}_1^*, \lim_{t\rightarrow\infty}\dot{w}_2^{[i]}(t)=\dot{w}_2^*$.
		\end{itemize}
	\end{problem}
	\subsection{Guiding vector field for versatile distributed maneuvering}	
	\subsubsection{Problem analysis}\label{1d_manifold_pro_ana}First review the definition of functions $\phi_{j}^{[i]}$ in the previous section. To accomplish the first objective of Problem~\ref{problem1}, one needs to ensure the validity of $\lim_{t\rightarrow\infty}\lVert x_j^{[i]}(t)-x_{\mathrm{d}j}^{[i]}(t)-x_{\mathrm{r}j}^{[i]}(t)\rVert = 0$ for $i\in\mathbb{Z}_1^N$ and $j\in\mathbb{Z}_1^n$. Given the parameterized $x_{\mathrm{d}j}^{[i]}$ and $x_{\mathrm{r}j}^{[i]}$, the first objective is accomplished if $\lim_{t\rightarrow\infty}\lVert x_j^{[i]}-f_j^{[i]}(w_1^{[i]})-g_j^{[i]}(w_2^{[i]})\rVert\rightarrow 0$ for $i\in\mathbb{Z}_1^N$ and $j\in\mathbb{Z}_1^n$. Note that $f_j^{[i]}(w_1^{[i]})$ and $g_j^{[i]}(w_2^{[i]})$ correspond to interception and target enclosing behavior, respectively. The second objective can be accomplished if the virtual coordinates converge to their desired values under the consensus theory\cite{mesbahi2010graph}. 
	\subsubsection{Algorithm design}
	Based on the above analysis, we notice that the condition for solving Problem~\ref{problem1} can be equivalently transformed to make the $i$-th robot follow a path on the composite 2D manifold, which is described by $n$ equations, i.e.,
	\begin{equation}\label{parametric_equ_circum}\resizebox{0.9\hsize}{!}{$
	    \begin{aligned}
	        x_{1}^{[i]}=f_1^{[i]}(w_1^{[i]})+g_1^{[i]}(w_2^{[i]}),\cdots,x_{n}^{[i]}=f_n^{[i]}(w_1^{[i]})+g_n^{[i]}(w_2^{[i]}).
	    \end{aligned}$}
	\end{equation}
	
	For the $i$-th robot, the desired composite trajectory is described by a zero-level set, i.e.,
	\begin{equation*}
		\mathcal{P}^{[i]}:=\{\boldsymbol{\xi^{[i]}}\in\mathbb{R}^{n+2}:\phi_j^{[i]}(\boldsymbol{\xi^{[i]}})=0,j\in\mathbb{Z}_1^n\},
	\end{equation*}
	where $\boldsymbol{\xi^{[i]}}:=(x_1^{[i]},\cdots,x_n^{[i]},w_1^{[i]},w_2^{[i]})\in\mathbb{R}^{n+2}$ is a \emph{generalized coordinate} with two virtual coordinates $w_1^{[i]}$ and $w_2^{[i]}$; and $\phi_j^{[i]}(\boldsymbol{\xi^{[i]}})=x_j^{[i]}-f_j^{[i]}(w_1^{[i]})-g_j^{[i]}(w_2^{[i]})$.
	
	Let 
	\begin{equation*}
		\boldsymbol{\Phi^{[i]}}(\boldsymbol{\xi^{[i]}}):=(\phi_1^{[i]}(\boldsymbol{\xi^{[i]}}),\cdots,\phi_n^{[i]}(\boldsymbol{\xi^{[i]}}))^\top\in\mathbb{R}^n
	\end{equation*}
	quantify the distance of robot $i$ to the desired composite trajectory on a 2D manifold. Since two virtual coordinates are included in the generalized coordinate, an extra vector is required in the wedge product operation, i.e.,
		$\boldsymbol{\lambda}=(\lambda_1,\cdots,\lambda_{n},\lambda_{w_2},\lambda_{w_1})^\top\in\mathbb{R}^{n+2}$,
	where $\lambda_{w_1},\lambda_{w_2}\in\mathbb{R}$ are related to the speed of virtual coordinates (see Eq.~\eqref{expansion_higherGVF}). One needs to set ${\lambda}_{w_1}$ and $\lambda_{w_2}$ to reach the desired interception and enclosing speeds. It is also straightforward to introduce two extra vectors for 3D composite manifold cases and so forth. The singularity-free guiding vector field is obtained by
	\begin{equation}\label{path_path_cross}
		\begin{aligned}
			\boldsymbol{\chi^{[i]}}(\boldsymbol{\xi^{[i]}})= & \wedge(\boldsymbol{\nabla \phi_1^{[i]}}(\boldsymbol{\xi^{[i]}}), \ldots, \boldsymbol{\nabla \phi_n^{[i]}}(\boldsymbol{\xi^{[i]}}), \boldsymbol{\lambda}) \\
			& -\sum_{j=1}^n k_j^{[i]} \phi_j^{[i]}(\boldsymbol{\xi^{[i]}}) \boldsymbol{\nabla \phi_j^{[i]}}(\boldsymbol{\xi^{[i]}}),
		\end{aligned}
	\end{equation}
	where $\wedge(\cdot)$ denotes the wedge product operation, $k_j^{[i]}$ are constant gains, $\boldsymbol{\nabla \phi_j^{[i]}}(\boldsymbol{\xi^{[i]}})$ is the gradient of $\phi_j^{[i]}(\boldsymbol{\xi^{[i]}})$ with respect to the generalized coordinate $\boldsymbol{\xi^{[i]}}$, i.e.,
	\begin{equation}\label{path_nabla_phi}\resizebox{0.85\hsize}{!}{$
			\begin{aligned}
				&\boldsymbol{\nabla \phi_j^{[i]}}(\boldsymbol{\xi^{[i]}})\\&=\partial{\phi_j^{[i]}(\boldsymbol{\xi^{[i]})}}/\partial\boldsymbol{\xi^{[i]}}\\&=(\dfrac{\partial\phi_j^{[i]}(\boldsymbol{\xi^{[i]}})}{\partial x_1^{[i]}},\cdots,\dfrac{\partial\phi_j^{[i]}(\boldsymbol{\xi^{[i]}})}{\partial x_n^{[i]}},\dfrac{\partial\phi_j^{[i]}(\boldsymbol{\xi^{[i]}})}{\partial w_1^{[i]}},\dfrac{\partial\phi_j^{[i]}(\boldsymbol{\xi^{[i]}})}{\partial w_2^{[i]}})^\top\\
				&\begin{array}{l}=\left(0,\cdots,1,\cdots,0,-{f_j^{[i]'}}(w_1^{[i]}),-{g_j^{[i]'}}(w_2^{[i]})\right)^\top\end{array},
			\end{aligned}$}
	\end{equation}
	where ${f_j^{[i]'}}(w_1^{[i]}):=\partial{f_j^{[i]}}(w_1^{[i]})/\partial{w_1^{[i]}}$ and ${g_j^{[i]'}}(w_2^{[i]}):=\partial{g_j^{[i]}}(w_2^{[i]})/\partial{w_2^{[i]}}$. For simplicity, ${f_j^{[i]'}}(w_1^{[i]})$ and ${g_j^{[i]'}}(w_2^{[i]})$ are denoted by ${f_j^{[i]'}}$ and ${g_j^{[i]'}}$, respectively.
	According to Eqs.~\eqref{path_path_cross} and~\eqref{path_nabla_phi}, we have the vector field of the (2D) manifold navigation part, i.e.,
	\begin{equation}\label{expansion_higherGVF}\resizebox{0.85\hsize}{!}{$
			\begin{aligned}
				&\boldsymbol{\chi^{[i]}}(\boldsymbol{\xi^{[i]}})=\\&(-1)^n\left[\begin{array}{c}
					\lambda_{w_1}f_1^{[i]'}-\lambda_{w_2}g_1^{[i]'} \\
					\vdots \\ 
					\lambda_{w_1}f_n^{[i]'}-\lambda_{w_2}g_n^{[i]'} \\
					\lambda_{w_1} \\
					-\lambda_{w_2}
				\end{array}\right]+ \left[\begin{array}{c}
					-k_1^{[i]}\phi_1^{[i]} \\
					\vdots \\
					-k_n^{[i]}\phi_n^{[i]} \\
					\sum_{j=1}^n k_j^{[i]} \phi_j^{[i]}f_j^{[i]'} \\
					\sum_{j=1}^n k_j^{[i]} \phi_j^{[i]}g_j^{[i]'}
				\end{array}\right].
			\end{aligned}$}
	\end{equation}
	
	Regarding ${}^{\mathbf{cr}}\boldsymbol {\chi^{[i]}_1}(\boldsymbol{w_1^{[\cdot]}})$ and $ {}^{\mathbf{cr}}\boldsymbol {\chi^{[i]}_2}(\boldsymbol{w_2^{[\cdot]}})$ of the coordinating part, they are obtained by~\eqref{CR_sur}. Therefore, we have the DGVF for versatile distributed maneuvering on the composite 2D manifold, i.e.,
	\begin{equation}\label{final_1d_manifold_gvf}
		\begin{aligned}
			\boldsymbol{\mathfrak{X}^{[i]}}(\boldsymbol{\mathcal{\xi}^{[i]}})=\boldsymbol{ \chi^{[i]}}(\boldsymbol{\xi^{[i]}})+k_{c1}{}^{\mathbf{cr}}\boldsymbol{\chi_1^{[i]}}( \boldsymbol{w_1^{[\cdot]}})+k_{c2}{}^{\mathbf{cr}}\boldsymbol{\chi_2^{[i]}}(\boldsymbol{w_2^{[\cdot]}}),
		\end{aligned}
	\end{equation}
	where $k_{c1},k_{c2}>0$ are coefficients. 

    Two standard and mild assumptions are imposed.
	\begin{assumption}\label{assumption1}
		The communication graph $\mathcal{G}=(\mathcal{V},\mathcal{E})$ contains a directed spanning tree.
	\end{assumption}
	\begin{assumption}\label{assumption2}
		The first derivatives $f_j^{[i]'}$, $g_j^{[i]'}$, and the second derivatives  $f_j^{[i]''}(w_1^{[i]}):=\frac{\partial^2{f_j^{[i]}}(w_1^{[i]})}{\partial{w_1^{[i]}}\partial{w_1^{[i]}}}$, $g_j^{[i]''}(w_2^{[i]}):=\frac{\partial^2{f_j^{[i]}}(w_2^{[i]})}{\partial{w_2^{[i]}}\partial{w_2^{[i]}}}$, are all bounded for  $j\in\mathbb{Z}_1^n$.
	\end{assumption}
 
    Based on Assumptions~\ref{assumption1},\ref{assumption2} and setting $\lambda_{w_1}=(-1)^{n+2}\dot{w}_1^*$, $\lambda_{w_2}=(-1)^{n+1}\dot{w}_2^*$, the designed DGVF, under specific parametric settings for distributed maneuvering, globally solves Problem 1. The proof is in the journal version \cite{yang18}.

	\section{Controller Design For Unicycle Robots}
	If a single-integrator model governs the robot's motion, the designed DGVF can directly generate robot controls. However, taking a unicycle-type robot as an example, one needs to generate feasible control inputs based on the guiding vector field $\boldsymbol{\mathfrak{X}}$. This section will design a controller for unicycle robots traveling at a varying speed $v$. The model for the $i$-th robot is described by
	\begin{equation}\label{unicycle_model}
		\dot{x}^{[i]}_1=v^{[i]}\cos\theta^{[i]}, \dot{x}^{[i]}_2=v^{[i]}\sin\theta^{[i]},\dot{x}^{[i]}_3=u_z^{[i]},\dot{\theta}^{[i]}=u_\theta^{[i]},
	\end{equation}
	where ${x}^{[i]}_1$, ${x}^{[i]}_2$, ${x}^{[i]}_3$ are coordinates of the robot's mass center, $\theta^{[i]}$ is the yaw angle, and $v^{[i]}$, $u_z^{[i]}$, $u_\theta^{[i]}$ are controls.
	
	In the following, $v^{[i]}$, $u_z^{[i]}$, $u_\theta^{[i]}$ are to be designed, respectively. Regarding $v^{[i]}$, it is calculated by $v^{[i]}:=({(\mathfrak{X}_1^{[i]})^2+(\mathfrak{X}_2^{[i]})^2})^{\frac{1}{2}}$, where $\mathfrak{X}_1^{[i]}$ and $\mathfrak{X}_2^{[i]}$ are first two elements of $\boldsymbol{\mathfrak{X}^{[i]}}$. Regarding $u_z^{[i]}$, it is used to let the robot climb or descend, and defined by
	\begin{equation} \label{uz}
		u_z^{[i]}=v^{[i]}\mathfrak{X}_3^{[i]}/({(\mathfrak{X}_1^{[i]})^2+(\mathfrak{X}_2^{[i]})^2})^{\frac{1}{2}},
	\end{equation}
	Regarding $u_\theta^{[i]}$, we employ a proportional control manner by

	\begin{equation}\label{u_theta}
		u_\theta^{[i]} =k_\theta(\theta_d^{[i]}-\theta^{[i]}),
	\end{equation}
	where $\theta_d^{[i]}:=\mathrm{atan2}(\frak{X}_2^{[i]},\frak{X}_1^{[i]})$ and $k_{\theta}>0$ is the coefficient, $\mathrm{atan2}(\cdot,\cdot):\mathbb{R}\times\mathbb{R}\rightarrow [-\pi,\pi)$ maps two components $\frak{X}_2^{[i]}$ and $\frak{X}_1^{[i]}$ to the angle to the X-axis. 
     The control inputs are validated by the following result.
	\begin{theorem}\label{theorem2}
		Assume that $(\mathfrak{X}_1^{[i]}(\boldsymbol{\xi^{[i]}}))^2+(\mathfrak{X}_2^{[i]}(\boldsymbol{\xi^{[i]}}))^2>0$ for $i\in\mathbb{Z}_1^N$ and $\boldsymbol{\xi^{[i]}}\in\mathbb{R}^5$, the angle difference vanishes asymptotically using control inputs~\eqref{u_theta}, \eqref{uz} in~\eqref{unicycle_model}. 
	\end{theorem}
	\begin{proof}
		First, define the angle difference by $e:=\theta_d^{[i]}-\theta^{[i]}$. Choose the Lyapunov function candidate $V=1/2e^2$, and the time derivative is
		\begin{equation}
			\dot{V}=e\dot{e}=-(\theta_d^{[i]}-\theta^{[i]})\dot{\theta}^{[i]}\overset{\eqref{u_theta}}{=}-k_\theta(\theta_d^{[i]}-\theta^{[i]})^2<0.
		\end{equation}
		Since the derivative of the Lyapunov candidate is negative definite, the angle difference $\lim_{t\rightarrow\infty}e(t)\rightarrow 0$ after sustainably applying  control~\eqref{u_theta} to~\eqref{unicycle_model}. 
	\end{proof}

	\begin{remark}
		One can adopt existing techniques for collision avoidance (e.g., \cite{wang2017safety, yao2022guiding}). 
        The result is shown in the supplementary video, but the theoretical analysis will be provided in the journal version \cite{yang18}.
		$\hfill\blacktriangleleft$
	\end{remark}
	\section{Simulation and Experimental Results}\label{simulation_experiments}
	This section conducts simulations and real-world experiments on multiple mobile robots to demonstrate the proposed approach's effectiveness and validate the theoretical results.
	\subsection{Simulations}
	The robot's communication topology is described by undirected graphs; i.e., the robots exchange virtual coordinate information with their neighboring robots. To facilitate the validation of the theoretical results, the target's motion is preset in enclosing and circumnavigation simulations; i.e., its velocity and path are known. Note that our approach also works under the target's \emph{unknown movement}, which has been validated by real-world experimental results in Sec.~\ref{Sec_exp}.
    
	\subsubsection{Model: single-integrator; Motion: formation and  enclosing maneuvering}

	We let a group of $82$ robots realize formation tracking on the 2D plane, with the randomly initialized positional coordinates in $\mathbb{R}^3$. The interception manifold $\mathcal{P}_f^{[i]}:=\{\boldsymbol{\xi_\mathrm{d}^{[i]}}\in\mathbb{R}^{n+1}:x_{\mathrm{d}1}^{[i]}-w_1^{[i]}=0,x_{\mathrm{d}2}^{[i]}-5\sin (\pi w_1^{[i]}/50)=0,x_{\mathrm{d}3}^{[i]}=0\},i\in\mathbb{Z}_1^{82}$ makes the robots follow the sine-wave path, and the enclosing manifold $\mathcal{P}_g^{[i]}:=\{\boldsymbol{\xi_\mathrm{r}^{[i]}}\in\mathbb{R}^{n+1}:
	x_{\mathrm{r}1}^{[i]}-\alpha^{[i]}=0, x_{\mathrm{r}2}^{[i]}-\beta^{[i]}=0, x_{\mathrm{r}3}^{[i]}=0\},i\in\mathbb{Z}_1^{82}$ makes the robots form the ``{\tt ICRA25}" shape,
	where $\alpha^{[i]},\beta^{[i]}\in\mathbb{R}$ are the desired displacements between neighboring robots.
	We also set $\boldsymbol{\Delta_{1|1:13}^{*}}=\boldsymbol{0}, \boldsymbol{\Delta_{1|14:20}^{*}}=\boldsymbol{10}, \boldsymbol{\Delta_{1|21:38}^{*}}=\boldsymbol{20},\boldsymbol{\Delta_{1|39:48}^{*}}=\boldsymbol{30}, \boldsymbol{\Delta_{1|49:65}^{*}}=\boldsymbol{40}, \boldsymbol{\Delta_{1|66:82}^{*}}=\boldsymbol{50}$ and $\boldsymbol{\Delta_{2}^{*}}=\boldsymbol{0}$, where $a:b$ represents the index range from $a$ to $b$. The desired path-following speed is set as ${\dot{w}_1^*}={3} \rm{m/s}$ and ${\dot{w}_2^*}$ will not work since $\mathcal{P}_g^{[i]}$ is not parameterized by $w_2^{[i]}$.
	The formation tracking results are shown in Fig.~\ref{fig_two_d_formation}. 
     In the second simulation, we let $10$ robots enclose the target. The interception manifold $\mathcal{P}_f^{[i]}:=\{\boldsymbol{\xi_\mathrm{d}^{[i]}}\in\mathbb{R}^{n+1}:x_{\mathrm{d}1}^{[i]}-30\cos w_1^{[i]}=0, x_{\mathrm{d}2}^{[i]}-30\sin w_1^{[i]}=0, x_{\mathrm{d}3}^{[i]}=0\},i\in\mathbb{Z}_1^{10}$ makes robots intercept the target moving along the circular path, and the enclosing manifold $\mathcal{P}_g^{[i]}:=\{\boldsymbol{\xi_\mathrm{r}^{[i]}}\in\mathbb{R}^{n+1}:x_{\mathrm{r}1}^{[i]}-10\sin w_2^{[i]}\tan^{-1}(w_2^{[i]})=0,x_{\mathrm{r}2}^{[i]}=0,	x_{\mathrm{r}3}^{[i]}-10\cos w_2^{[i]}=0\},i\in\mathbb{Z}_1^{10}$ makes the robots enclose on an ellipse-like path.
	We set $\boldsymbol{\Delta_1^*}=\boldsymbol{0}$, $\boldsymbol{\Delta_2^*}=[\frac{2\pi}{10},\frac{4\pi}{10},\cdots,2\pi]$ and speeds ${\dot{w}_1^*}={3} \rm{m/s}$ and ${\dot{w}_2^*}={3} \rm{m/s}$. The  target enclosing results are illustrated in Fig.~\ref{fig_1manifold_sur}. 
	\begin{figure}[!htbp]
		\centering\includegraphics[width=3.3in]{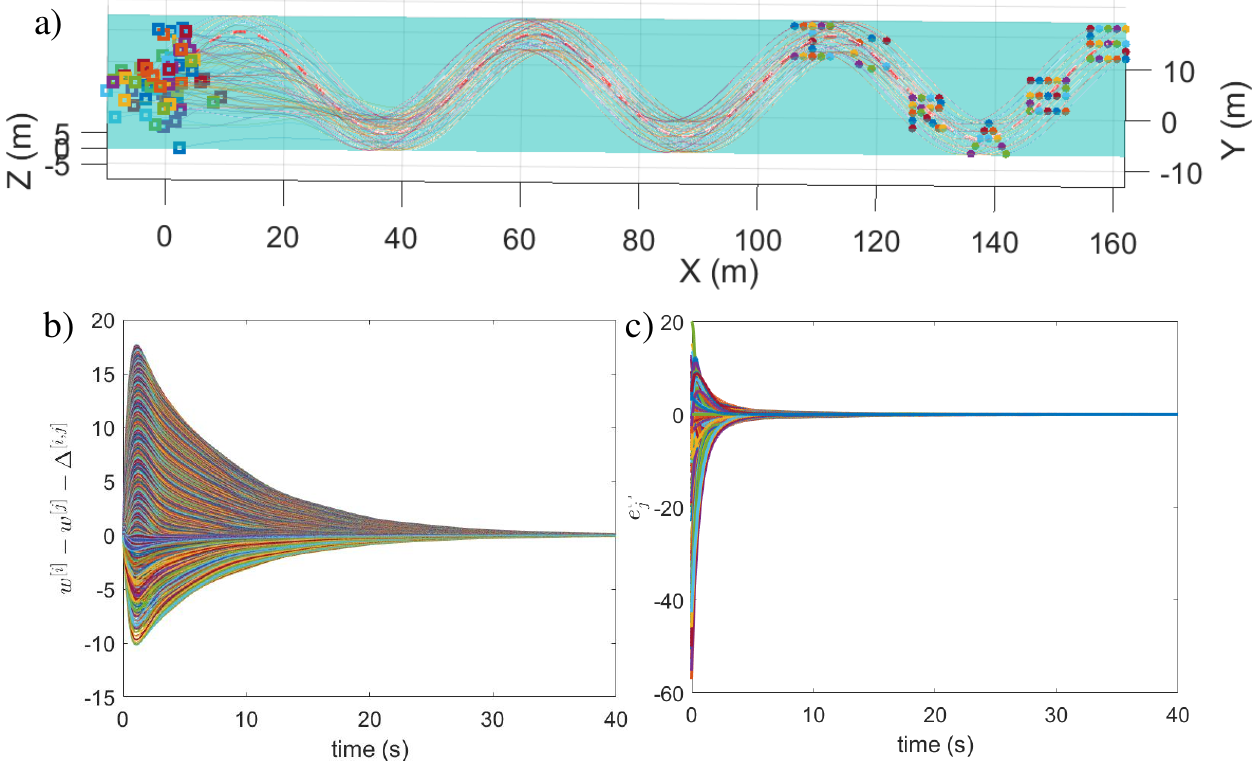}
		\caption{The first simulation results. The squares represent the initial positions of the robots, whereas the solid circles denote their final positions. The red dashed line is the target's path, and the thin lines are trajectories of $82$ robots. a) formation tracking. b) Coordination errors $w_1^{[i]}-w_1^{[j]}-\Delta_1^{[i,j]}$ and $w_2^{[i]}-w_2^{[j]}-\Delta_2^{[i,j]}$ converge to zero eventually, for $i,j\in\mathbb{Z}_1^{82}$, $i<j$. c) Formation errors $\phi_j^{[i]}$ for $i\in\mathbb{Z}_1^{82}$, $j\in\mathbb{Z}_1^3$.}
		\label{fig_two_d_formation}
	\end{figure}
	\begin{figure}[!htbp]
		\centering\includegraphics[width=3.0in]{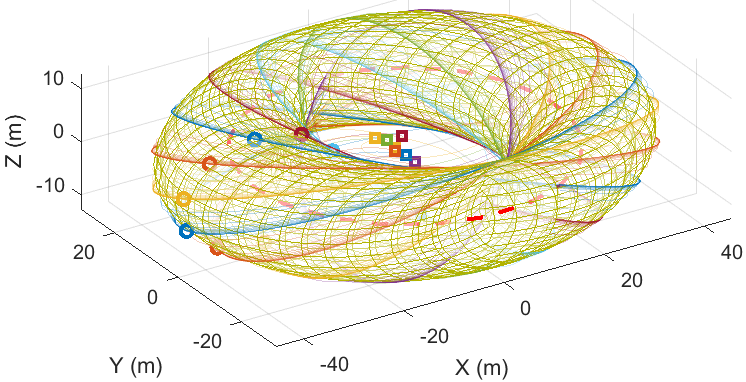}
		\caption{The second simulation results. The squares and circles represent the initial and final positions of the robots, respectively. The red dashed line is the target's path, and the thin lines are trajectories of $10$ robots.}
		\label{fig_1manifold_sur}
	\end{figure}

     \begin{figure}[!htbp]
		\centering\includegraphics[width=2.5in]{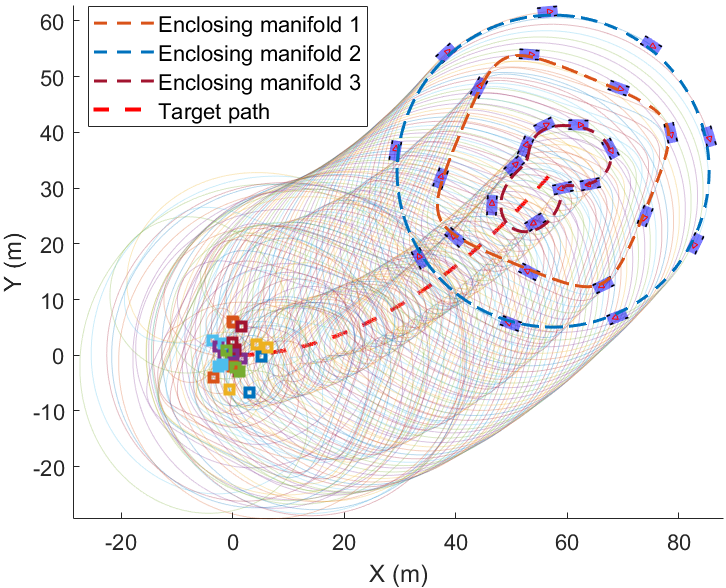}
		\caption{The third simulation results. The squares represent the initial positions of $27$ unicycle robots, whereas the blue rectangles are the robots at the final instant. The thin lines are the robots' trajectories. Solid lines with different colors denote the formed orbits at the final instant. 
		}
		\label{fig_1manifold_sym}
	\end{figure} 
	\subsubsection{Model: unicycle; Motion: circumnavigation}
	We let $27$ robots circumnavigate a moving target. The interception manifold $\mathcal{P}_f^{[i]}:=\{\boldsymbol{\xi_\mathrm{d}^{[i]}}\in\mathbb{R}^{n+1}:x_{\mathrm{d}1}^{[i]}-w_1^{[i]}=0,x_{\mathrm{d}2}^{[i]}-0.01(w_1^{[i]})^2=0,x_{\mathrm{d}3}^{[i]}=0\},i\in\mathbb{Z}_1^{27}$ makes the robots intercept the target moving along a curve line. The enclosing manifolds $\mathcal{P}_g^{[i]}:=\{\boldsymbol{\xi_\mathrm{r}^{[i]}}\in\mathbb{R}^{n+1}:x_{\mathrm{r}1}^{[i]}-8(1 + 0.4 \sin(2w_2^{[i]}) + 0.2\cos(3 w_2^{[i]}))\cos(w_2^{[i]})=0,x_{\mathrm{r}2}^{[i]}-8(1 + 0.4 \sin(2w_2^{[i]}) + 0.2\cos(3 w_2^{[i]}))\sin(w_2^{[i]})=0,x_{\mathrm{r}3}^{[i]}=0\}$, $i\in\mathbb{Z}_1^9$, $\mathcal{P}_g^{[i]}:=\{\boldsymbol{\xi_\mathrm{r}^{[i]}}\in\mathbb{R}^{n+1}:x_{\mathrm{r}1}^{[i]}-28\cos w_2^{[i]}=0,x_{\mathrm{r}2}^{[i]}-28\sin w_2^{[i]}=0,x_{\mathrm{r}3}^{[i]}=0\}$, $i\in\mathbb{Z}_{10}^{18}$, and $\mathcal{P}_g^{[i]}:=\{\boldsymbol{\xi_\mathrm{r}^{[i]}}\in\mathbb{R}^{n+1}:x_{\mathrm{r}1}^{[i]}-20(\cos(w_2^{[i]})+0.1\sin(3w_2^{[i]}))=0,x_{\mathrm{r}2}^{[i]}-20(\sin(w_2^{[i]})+0.1\cos(3w_2^{[i]}))=0,x_{\mathrm{r}3}^{[i]}=0\}$, $i\in\mathbb{Z}_{19}^{27}$ make the robots circumnavigate the target coordinately. We set $\boldsymbol{\Delta_1^*}=\boldsymbol{0}$, $\boldsymbol{\Delta_2^*}=(i-1)6\pi/27$ for $i\in\mathbb{Z}_1^{27}$ and the desired speeds $\dot{w}_1^*=3 \rm{m/s}$, $\dot{w}_2^*=2\mathrm{m/s}$. The target circumnavigation results are illustrated in Fig.~\ref{fig_1manifold_sym}. 
	\subsection{Real-world experiments}\label{Sec_exp}
	Experiments were performed on our custom-made multi-robot platform \cite{zhu2024dvrpmhsidynamicvisualizationresearch}. The robots' positions are acquired from the touch screen at a frequency of $200$ Hz. Based on the communication topology described by undirected graphs and the designed approach, the ground computer computes the saturated control inputs, which are transmitted to each robot at $30$ Hz.
 %
 %
 The robots are required to constitute a circular formation to circumnavigate a moving target. To this end, the interception manifold $\mathcal{P}_f^{[i]}:=\{\boldsymbol{\xi_\mathrm{d}^{[i]}}\in\mathbb{R}^{n+1}:x_{\mathrm{d}1}^{[i]}-x_{\mathrm{target}}=0,x_{\mathrm{d}2}^{[i]}-y_{\mathrm{target}}=0,x_{\mathrm{d}3}^{[i]}=0\}, i\in\mathbb{Z}_1^{5}$ makes the robots intercept the target, and the enclosing manifold $\mathcal{P}_g^{[i]}:=\{\boldsymbol{\xi_\mathrm{r}^{[i]}}\in\mathbb{R}^{n+1}:x_{\mathrm{r}1}^{[i]}-0.24\cos w_2^{[i]}=0,x_{\mathrm{r}2}^{[i]}-0.24\sin w_2^{[i]}=0,x_{\mathrm{r}3}^{[i]}=0\}, i\in\mathbb{Z}_1^{5}$ makes the robots circumnavigate on the circular path, where $(x_{\mathrm{target}},y_{\mathrm{target}})$ is the target's real-time position.
	We set  $\boldsymbol{\Delta_1^*}=\boldsymbol{0}$, $\boldsymbol{\Delta_2^*}=[0,\frac{2\pi}{5},\frac{4\pi}{5},\frac{6\pi}{5},\frac{8\pi}{5}]$, and the desired enclosing speed $\dot{w}_2^*=0.02 \mathrm{m/s}$. Since $\mathcal{P}_f^{[i]}$ is not parameterized by $w_1^{[i]}$, the $\dot{w}_1^*$ has no effect.
 The experimental results are shown in Fig.~\ref{fig_fig_mv_sr_line}-a). Then we set $\mathcal{P}_g^{[i]}:=\{\boldsymbol{\xi_\mathrm{r}^{[i]}}\in\mathbb{R}^{n+1}:x_{\mathrm{r}1}^{[i]}-0.24(\cos(w_2^{[i]})+0.2\sin(4w_2^{[i]}))=0,x_{\mathrm{r}2}^{[i]}-0.24(\sin(w_2^{[i]})+0.2\cos(4w_2^{[i]}))=0,x_{\mathrm{r}3}^{[i]}=0\}, i\in\mathbb{Z}_1^{5}$, to circumnavigate the target on the star-shaped orbit. The other settings are consistent with the first experimental task, and the results are presented in Fig.~\ref{fig_fig_mv_sr_line}-b).
	\begin{figure}[!htbp]
		\centering\includegraphics[width=2.7in]{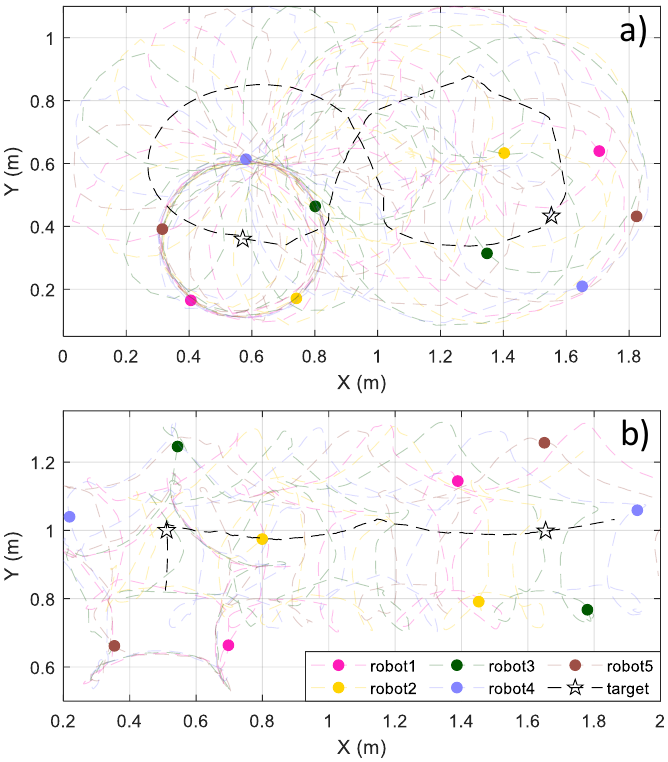}
		\caption{Five robots circumnavigate on the a) circular and b) star-shaped orbits. We label the positions of robots and the target at two instants.}
		\label{fig_fig_mv_sr_line}
	\end{figure}
 \section{Conclusion}\label{Conclusion}
	This paper proposes a versatile distributed maneuvering approach for multi-robot systems using guiding vector fields. The two intercepting and enclosing behaviors are parameterized by independent virtual coordinates and combined as a composite 2D manifold. Neighboring robots communicate the information of virtual coordinates to achieve motion coordination. We have presented assumptions and analyses for solving the formulated versatile maneuvering problem. Simulations of formation maneuvering, target enclosing, circumnavigation with single-integrator and unicycle models, and real-world experiments on a group of mobile robots demonstrate the effectiveness of the versatile approach. Note that the proposed approach can be straightforwardly generalized to combine more than two behaviors, and the resulting guiding vector field does not have singularities; i.e., deadlock is avoided. 
	Future research directions may include state estimation of the target and detailed theoretical analysis of collision avoidance among robots.

	\bibliographystyle{IEEEtran}
	\bibliography{IEEEabrv.bib}
\end{document}